\documentclass[10pt,journal,compsoc]{IEEEtran}

\usepackage{times}
\usepackage{helvet}
\usepackage{courier}
\usepackage{multicol}
\usepackage{braket}
\usepackage{array}
\usepackage{tabu}
\usepackage{diagbox}

\ifCLASSOPTIONcompsoc
  \usepackage[nocompress]{cite}
\else
  \usepackage{cite}
\fi

\usepackage{graphicx}
\usepackage{subfigure}

\usepackage{amsmath}
\usepackage{amsthm}
\usepackage{amssymb}
\newtheorem{theorem}{Theorem}
\newtheorem{remark}{Remark}

\usepackage{algorithm}
\usepackage{algorithmic}

\newcommand{\argminF}{\mathop{\mathrm{argmin}}\limits}

\usepackage{url}

\begin{document}
%

\title{A Secure Federated Transfer Learning Framework}
%
%
%
%

\author{Yang Liu,
        Yan Kang,
        Chaoping Xing,
        Tianjian Chen,
        Qiang Yang,~\IEEEmembership{Fellow,~IEEE}
\IEEEcompsocitemizethanks{

\IEEEcompsocthanksitem Yang Liu, Yan Kang and Tianjian Chen are with WeBank, Shenzhen, China. \protect\\

\IEEEcompsocthanksitem Chanping Xing is with the Shanghai Jiao Tong University, Shanghai China.\protect\\

\IEEEcompsocthanksitem Qiang Yang is with the Hong Kong University of Science and Technology, Hong Kong,
China.\protect\\

}
}

%
%


\markboth{\tiny This is the author's version of an article that has been published in this journal. Changes were made to this version prior to publication. DOI: 10.1109/MIS.2020.2988525}{}
\IEEEtitleabstractindextext{%
\begin{abstract}
Machine learning relies on the availability of vast amounts of data for training. However, in reality, data are mostly scattered across different organizations and cannot be easily integrated due to many legal and practical constraints. To address this important challenge in the field of machine learning, we introduce a new technique and framework, known as federated transfer learning (FTL), to improve statistical modeling under a data federation. FTL allows knowledge to be shared without compromising user privacy and enables complementary knowledge to be transferred across domains in a data federation, thereby enabling a target-domain party to build flexible and effective models by leveraging rich labels from a source domain. This framework requires minimal modifications to the existing model structure and provides the same level of accuracy as the non-privacy-preserving transfer learning. It is flexible and can be effectively adapted to various secure multi-party machine learning tasks.
\end{abstract}
}


\IEEEpubid{\scriptsize Copyright (c) 2020 IEEE. Personal use is permitted. For any other purposes, permission must be obtained from the IEEE by emailing pubs-permissions@ieee.org.}

\maketitle

\IEEEdisplaynontitleabstractindextext

%
\IEEEpeerreviewmaketitle

\IEEEraisesectionheading{
\section{Introduction}\label{sec:introduction}}

%
%
%
%
\IEEEPARstart{R}{ecent} Artificial Intelligence (AI) achievements have been depending on the availability of massive amounts of labeled data. For example, AlphaGo has been trained using a dataset containing 30 million moves from 160,000 actual games. The ImageNet dataset has over 14 million images. However, across various industries, most applications only have access to small or poor quality datasets. Labeling data is very expensive, especially in fields which require human expertise and domain knowledge. In addition, data needed for a specific task may not all be stored in one place. Many organizations may only have unlabeled data, and some other organizations may have very limited amounts of labels. It has been increasingly difficult from a legislative perspective for organizations to combine their data, too. For example, General Data Protection Regulation (GDPR) \cite{regulation2016general}, a new bill introduced by the European Union, contains many terms that protect user privacy and prohibit organizations from exchanging data without explicit user approval. How to enable the large number of businesses and applications that have only small data (few samples and features) or weak supervision (few labels) to build effective and accurate AI models while complying with data privacy and security laws is a difficult challenge.

To address this challenge, Google introduced a federated learning (FL) system \cite{DBLP:journals/corr/McMahanMRA16} in which a global machine learning model is updated by a federation of distributed participants while keeping their data stored locally. Their framework requires all contributors share the same feature space. On the other hand, secure machine learning with data partitioned in the feature space has also been studied \cite{Gascn2016SecureLR}. These approaches are only applicable in the context of data with either common features or common samples under a federation. In reality, however, the set of such common entities may be small, making a federation less attractive and leaving the majority of the non-overlapping data under-utilized.

In this paper, we propose Federated Transfer Learning (FTL) to address the limitations of existing federated learning approaches. It leverages transfer learning \cite{Pan:2010:CSC:1772690.1772767} to provide solutions for the entire sample and feature space under a federation. Our contributions are as follows:
\begin{enumerate}
    \item We formalize the research problem of federated transfer learning in a privacy-preserving setting to provide solutions for federation problems beyond the scope of existing federated learning approaches;
    \item We provide an end-to-end solution to the proposed FTL problem and show that the performance of the proposed approach in terms of convergence and accuracy is comparable to non-privacy-preserving transfer learning; and
    \item We provide some novel approaches to incorporate additively homomorphic encryption (HE) and secret sharing using beaver triples into two-party computation (2PC) with neural networks under the FTL framework such that only minimal modifications to the neural network is required and the accuracy is almost lossless.
\end{enumerate}

\section{Related Work}
Recent years have witnessed a surge of studies on encrypted machine learning. For example, Google introduced a secure aggregation scheme to protect the privacy of aggregated user updates under their federarted learning framework \cite{Bonawitz:2017:PSA:3133956.3133982}.  CryptoNets \cite{cryptonets-applying-neural-networks-to-encrypted-data-with-high-throughput-and-accuracy} adapted neural network computations to work with data encrypted via Homomorphic Encryption (HE). SecureML \cite{DBLP:journals/iacr/MohasselZ17} is a multi-party computing scheme which uses secret-sharing and Yao's Garbled Circuit for encryption and supports collaborative training for linear regression, logistic regression and neural networks. 

Transfer learning aims to build an effective model for an application with a small dataset or limited labels in a target domain by leveraging knowledge from a different but related source domain. In recent years, there have been tremendous progress in applying transfer learning to various fields such as image classification and sentiment analysis. The performance of transfer learning relies on how related the domains are. Intuitively, parties in the same data federation are usually organizations from the same industry. Therefore, they can benefit more from knowledge propagation. To the best of our knowledge, FTL is the first framework to enable federated learning to benefit from transfer learning.

\section{Preliminaries and Security Definition}
Consider a source domain dataset $\mathcal{D}_A := \{(x_i^A,y_i^A)\}_{i=1}^{N_A}$
where $x_i^A \in R^a$ and $y_i^A \in \{+1,-1\}$ is the $i$ th label, and a target domain $\mathcal{D}_B :=\{x_j^B\}_{j=1}^{N_B}$ where $x_j^B \in R^a$. $\mathcal{D}_A$ and $\mathcal{D}_B$ are separately held by two private parties and cannot be exposed to each other legally. We assume that there exists a limited set of co-occurrence samples $\mathcal{D}_{AB} := \{(x_i^A,x_i^B)\}_{i=1}^{N_{AB}}$ and a small set of labels for data from domain B in party A's dataset: $\mathcal{D}_c:=\{(x_i^B,y_i^A)\}_{i=1}^{N_c}$, where $N_c$ is the number of available target labels. Without loss of generality, we assume all labels are in party A, but all the deduction here can be adapted to the case where labels exist in party B. One can find the set of commonly shared sample IDs in a privacy-preserving setting by masking data IDs with encryption techniques (e.g., the RSA scheme). We utilize the RSA Intersection module of the FATE\footnote{https://github.com/FederatedAI/FATE} framework to align co-occurrence samples of the two parties.
Given the above setting, the objective is for the two parities to build a transfer learning model to predict labels for the target-domain party accurately without exposing data to each other.

In this paper, we adopt a security definition in which all parties are \textit{honest-but-curious}. We assume a threat model with a semi-honest adversary $\mathcal{D}$ who can corrupt at most one of the two data clients. For a protocol $P$ performing $(O_A,O_B)=P(I_A,I_B)$, where $O_A$ and $O_B$ are party A's and party B's outputs and $I_A$ and $I_B$ are their inputs, respectively, $P$ is secure against A if there exists an infinite number of $(I'_B,O'_B)$ pairs such that $(O_A,O'_B)=P(I_A,I'_B)$. It allows flexible control of information disclosure compared to complete zero knowledge security.

\section{The Proposed Approach}
In this section, we introduce our proposed transfer learning model.

Deep neural networks have been widely adopted in transfer learning to find implicit transfer mechanisms. Here, we explore a general scenario in which hidden representations of A and B are produced by two neural networks $u_i^A=\it{Net}^A(x_i^A)$ and $u_i^B=\it{Net}^B(x_i^B)$, where $u^A\in \mathbb{R}^{N_A\times d}$ and $u^B\in \mathbb{R}^{N_B\times d}$, and $\it{d}$ is the dimension of the hidden representation layer. The neural networks $\it{Net}^A$ and $\it{Net}^B$ serve as feature transformation functions that project the source features of party A and party B into a common feature subspace in which knowledge can be transferred between the two parties. Any other symmetric feature transformation techniques can be applied here to form the common feature subspace. However, neural networks can help us build an end-to-end solution to the proposed FTL problem.

To label the target domain, a general approach is to introduce a prediction function $\varphi(u_j^B) = \varphi(u_1^A,y_1^A...u_{N_A}^A,y_{N_A}^A,u_j^B)$. Without losing much generality, we assume $\varphi(u_j^B)$ is linearly separable. That is, $\varphi(u_j^B)=\Phi^A\mathcal{G}(u_j^B)$. In our experiment, we use a translator function, $\varphi(u_j^B)=\frac{1}{N_A}\sum_i^{N_A}y_i^{A}u_i^{A}(u_j^{B})'$, where $\Phi^A=\frac{1}{N_A}\sum_i^{N_A}y_i^{A}u_i^{A}$ and $\mathcal{G}(u_j^B)=(u_j^{B})'$. We can then write the training objective using the available labeled set as:
\begin{equation}\label{l1}
\argminF_{\Theta^{A},\Theta^{B}}\mathcal{L}_1=\sum_{i=1}^{N_{c}}\ell_1(y_i^{A},\varphi(u_i^{B}))
\end{equation}
where $\Theta^A$, $\Theta^B$ are training parameters of $\it{Net}^A$ and $\it{Net}^B$, respectively. Let $L_A$ and $L_B$ be the number of layers for $\it{Net}^A$ and $\it{Net}^B$, respectively. Then, $\Theta^A=\{\theta_l^A\}_{l=1}^{L_A}$, $\Theta^B=\{\theta_l^B\}_{l=1}^{L_B}$ where $\theta_l^A$ and $\theta_l^B$ are the training parameters for the $l$th layer. $\ell_1$ denotes the loss function. For logistic loss, $\ell_1(y,\varphi)=\log(1+\exp(-y\varphi))$.

In addition, we also aim to minimize the alignment loss between A and B in order to achieve feature transfer learning in a federated learning setting:
\begin{equation}
\argminF_{\Theta^{A},\Theta^{B}}\mathcal{L}_2=\sum_{i=1}^{N_{AB}}\ell_2(u_i^{A},u_i^{B})
\end{equation}
where $\ell_2$ denotes the alignment loss. Typical alignment losses can be $-u_i^{A}(u_i^{B})'$ or $||u_i^A-u_i^B||_F^2$. For simplicity, we assume that it can be expressed in the form of $\ell_2(u_i^{A},u_i^{B})=\ell_2^A(u_i^{A})+\ell_2^B(u_i^{B})+\kappa u_i^{A}(u_i^{B})'$, where $\kappa $ is a constant.

The final objective function is:
\begin{equation}\label{loss}
\argminF_{\Theta^{A},\Theta^{B}}\mathcal{L}=\mathcal{L}_1+\gamma\mathcal{L}_2+\frac{\lambda}{2}(\mathcal{L}_3^A+\mathcal{L}_3^B)
\end{equation}
where $\gamma$ and $\lambda$ are the weight parameters, and $\mathcal{L}_3 ^A=\sum_l^{L_A}||\theta_l^A||_F^2$ and $\mathcal{L}_3 ^B=\sum_l^{L_B}||\theta_l^B||_F^2$ are the regularization terms. Now we focus on obtaining the gradients for updating $\Theta^A$, $\Theta^B$ in back propagation. For $i \in \{A,B\}$, we have:
\begin{equation}\label{grad}
\frac{\partial\mathcal{L}}{\partial\theta_l^i}=\frac{\partial\mathcal{L}_1}{\partial\theta_l^i}+\gamma\frac{\partial\mathcal{L}_2}{\partial\theta_l^i}+\lambda\theta_l^i.
\end{equation}



Under the assumption that A and B are not allowed to expose their raw data, a privacy-preserving approach needs to be developed to compute equations (\ref{loss}) and (\ref{grad}). Here, we adopt a second order Taylor approximation for logistic loss:
\begin{equation}\label{taylorloss}
    \ell_1(y,\varphi) \approx \ell_1(y,0)+\frac{1}{2}C(y)\varphi+\frac{1}{8}D(y)\varphi^2
\end{equation}
and the gradient is:
\begin{equation}\label{taylorgrad}
\frac{\partial\ell_1}{\partial\varphi}=\frac{1}{2}C(y)+\frac{1}{4}D(y)\varphi.
\end{equation}

where, $C(y)=-y$, $D(y)=y^2$.

In the following two sections, we will discuss two alternative constructions of the secure FTL protocol: the first one is leveraging additively homomorphic encryption, and the second one is utilizing the secret sharing based on beaver triples. We carefully design the FTL protocol such that only minimal information needs to be encrypted or secretly shared between parties. Besides, the FTL protocol is designed to be compatible with other homomorphic encryption and secret sharing schemes with minimal modifications. 

\section{FTL using Homomorphic Encryption}
Additively Homomorphic Encryption and polynomial approximations have been widely used for privacy-preserving machine learning. Applying equations (\ref{taylorloss}) and (\ref{taylorgrad}), and additively homomorphic encryption (denoted as $[[\cdot]]$), we obtain the privacy preserved loss function and the corresponding gradients for the two domains as:
\begin{equation}\label{encryptloss}
\begin{split}
[[\mathcal{L}]]&=\sum_i^{N_{c}}([[\ell_1(y_i^A,0)]]+\frac{1}{2}C(y_i^A)\Phi^A[[\mathcal{G}(u_i^B)]]\\
&+\frac{1}{8}D(y_i^A)\Phi^A[[(\mathcal{G}(u_i^B))'\mathcal{G}(u_i^B)]](\Phi^A)')\\&
+\gamma\sum_i^{N_{AB}}([[\ell_2^B(u_i^B)]]+[[\ell_2^A(u_i^A)]]+\kappa u_i^A[[(u_i^B)']])\\&+[[\frac{\lambda}{2}\mathcal{L}_3^A]]+[[\frac{\lambda}{2}\mathcal{L}_3^B]],
\end{split}
\end{equation}

\begin{equation}\label{encryptgradb}
\begin{split}
[[\frac{\partial\mathcal{L}}{\partial\theta_l^B}]] &=\sum_i^{N_{c}}\frac{\partial(\mathcal{G}(u_i^B))'\mathcal{G}(u_i^B)}{\partial u_i^B}[[(\frac{1}{8}D(y_i^A)(\Phi^A)'\Phi^A]]\frac{\partial  u_i^B}{\partial\theta_l^B}\\
&+\sum_i^{N_{c}}[[\frac{1}{2}C(y_i^A)\Phi^A]]\frac{\partial \mathcal{G}(u_i^B)}{\partial u_i^B}\frac{\partial u_i^B}{\partial\theta_l^B} \\& + \sum_i^{N_{AB}}([[\gamma \kappa u_i^A]]\frac{\partial  u_i^B}{\partial\theta_l^B}+[[\gamma\frac{\partial\ell_2^B(u_i^B)}{\partial\theta_l^B}]])+[[\lambda\theta_l^B]],
\end{split}
\end{equation}

\begin{equation}\label{encryptgrada}
\begin{split}
[[\frac{\partial\mathcal{L}}{\partial\theta_l^A}]] &=\sum_j^{N_{A}}\sum_i^{N_{c}}(\frac{1}{4}D(y_i^A)\Phi^A[[\mathcal{G}(u_i^B)'\mathcal{G}(u_i^B)]]\\
&+\frac{1}{2}C(y_i^A)[[\mathcal{G}(u_i^B)]])\frac{\partial \Phi^A}{\partial u_j^A}\frac{\partial u_j^A}{\partial\theta_l^A}\\
&+\gamma\sum_i^{N_{AB}}([[\kappa  u_i^B]]\frac{\partial u_i^A}{\partial\theta_l^A}+[[\frac{\partial \ell_2^A(u_i^A)}{\partial\theta_l^A}]])+[[\lambda\theta_l^A]].
\end{split}
\end{equation}

\subsection{FTL Algorithm - Homomorphic Encryption based}

With equations (\ref{encryptloss}), (\ref{encryptgradb}) and (\ref{encryptgrada}), we now design a federated algorithm for solving the transfer learning problem. See Figure \ref{al:1}. Let $[[\cdot]]_A$ and $[[\cdot]]_B$ be homomorphic encryption operators with public keys from A and B, respectively. A and B initialize and execute their respective neural networks $\it{Net^A}$ and $\it{Net^B}$ locally to obtain hidden representations ${u_i^A}$ and ${u_i^B}$. A then computes and encrypts components $\{\it{h}_k(u_i^A,y_i^A)\}_{k=1,2...K_A}$ and sends to B to assist calculations of gradients of $\it{Net^B}$. In the current scenario, $K_A=3$, $\it{h}_1^A(u_i^A,y_i^A)=\{[[\frac{1}{8}D(y_i^A)(\Phi^A)'(\Phi^A)]]_{A}\}_{i=1}^{N_c}$, $\it{h}_2^A(u_i^A,y_i^A)=\{[[\frac{1}{2}C(y_i^A)\Phi^A]]_{A}\}_{i=1}^{N_c}$, and $\it{h}_3^A(u_i^A,y_i^A)=\{[[\gamma \kappa u_i^A]]_{A}\}_{i=1}^{N_{AB}}$. Similarly, B then computes and encrypts components $\{\it{h}_k^B(u_i^B)\}_{k=1,2...K_B}$ and sends to A to assist calculations of gradients of $\it{Net^A}$ and loss $\mathcal{L}$. In the current scenario, $K_B=4$, $\it{h}_1^B(u_i^B)=\{[[(\mathcal{G}(u_i^B))'\mathcal{G}(u_i^B)]]_{B}\}_{i=1}^{N_c}$, $\it{h}_2^B(u_i^B)=\{[[\mathcal{G}(u_i^B)]]_{B}\}_{i=1}^{N_c}$, $\it{h}_3^B(u_i^B)=\{[[\kappa  u_i^B]]_{B}\}_{i=1}^{N_{AB}}$, and $\it{h}_4^B(u_i^B)=[[\frac{\lambda}{2}\mathcal{L}_3^B]]_{B}$.

\begin{figure}[!ht]

    \includegraphics[width=0.99\linewidth]{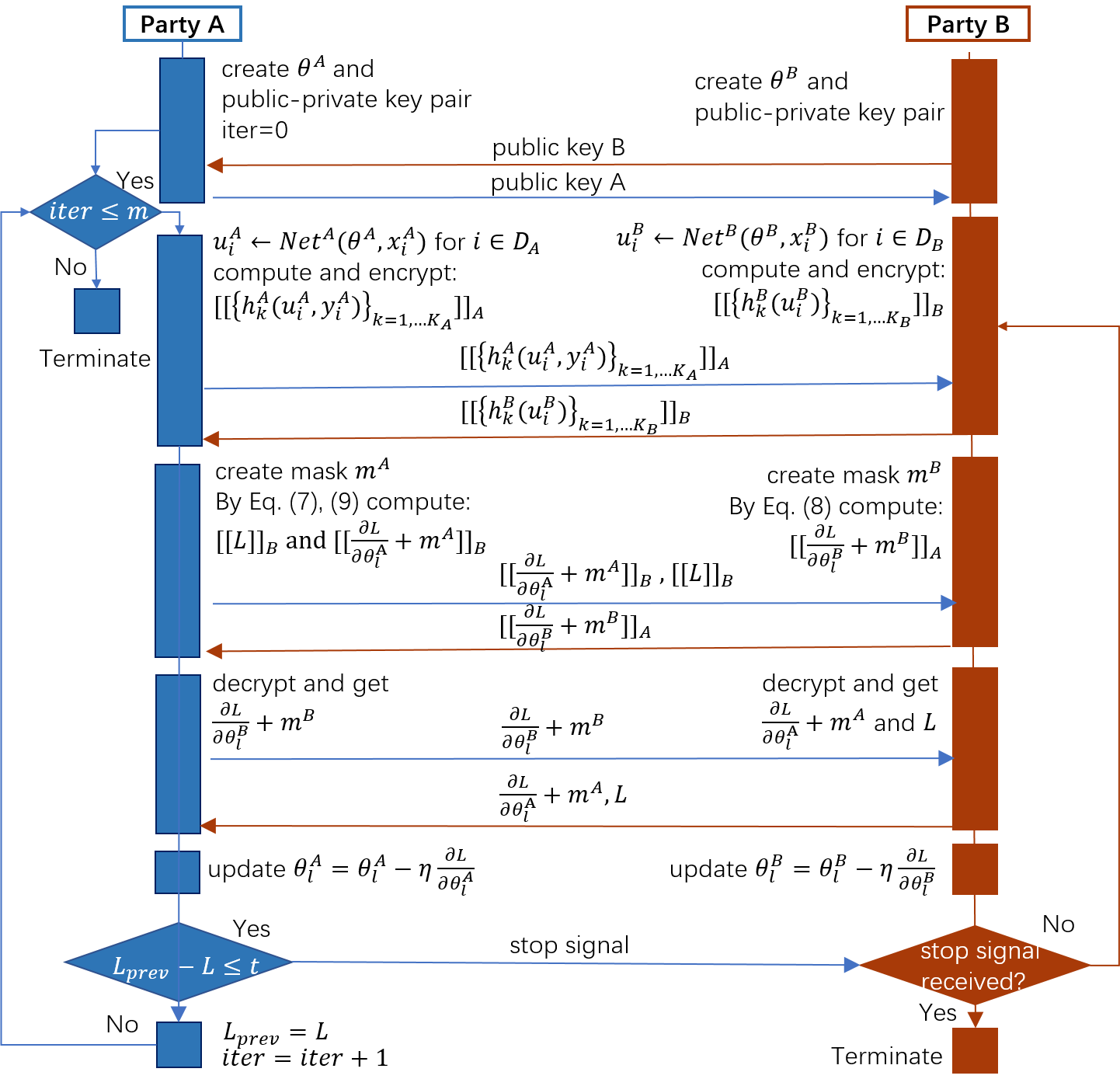}
    \caption{HE-based FTL algorithm workflow} 
    \label{al:1}
\end{figure}


To prevent A's and B's gradients from being exposed, A and B further mask each gradient with an encrypted random value. They then send the masked gradients and loss to each other and decrypt the values locally. A can send termination signals to B once the loss convergence condition is met. Otherwise, they unmask the gradients, update the weight parameters with their respective gradients, and move on to next iteration. Once the model is trained, FTL can provide predictions for unlabeled data from party B. Algorithm \ref{algo_pred} summaries the prediction process.


\begin{algorithm} [ht]
\caption{HE-based FTL: Prediction}
\begin{algorithmic}[1]\label{algo_pred}
\REQUIRE Model parameters $\Theta^A$, $\Theta^B$, $\{x_j^B\}_{j \in N_B}$
\STATE $\boldsymbol B$ do:
\STATE ${u_j^B} \xleftarrow{} \it{Net^B}(\Theta^B,x_j^B)$;
\STATE Encrypt $\{[[\mathcal{G}(u_j^B)]]_{B}\}_{j\in\{1,2,...,N_B\}}$ and send it to A;
\STATE $\boldsymbol A$ do:
\STATE Create a random mask $m^A$;
\STATE Compute $[[\varphi(u_j^B)]]_B=\Phi^A[[\mathcal{G}(u_j^B)]]_B$ and send $[[\varphi(u_j^B)+m^{A}]]_B$ to B;
\STATE $\boldsymbol B$ do:
\STATE Decrypt $\varphi(u_j^B)+m^A$ and send results to A;
\STATE $\boldsymbol A$ do:
\STATE Compute $\varphi(u_j^B)$ and $y_j^B$ and send $y_j^B$ to B;
\end{algorithmic}
\end{algorithm}

\subsection{Security Analysis}
\begin{theorem}
The protocol in the FTL training Algorithm (Figure \ref{al:1}) and Algorithm \ref{algo_pred} is secure under our security definition, provided that the underlying additively homomorphic encryption scheme is secure.
\end{theorem}

\begin{proof}
The training protocol in Figure \ref{al:1} and Algorithm \ref{algo_pred} do not reveal any information, because all A and B learned are the masked gradients. In each iteration, A and B create new random masks. The strong randomness and secrecy of the masks secure the information against the other party \cite{Du2004PrivacyPreservingMS}. During training, A learns its own gradients at each step, but this is not enough for A to learn any information from B based on the impossibility of solving $n$ equations with more than $n$ unknowns \cite{Du2004PrivacyPreservingMS}. In other words, there exist an infinite number of inputs from B that result in the same gradients A receives. Similarly, B cannot learn any information from A. Therefore, as long as the encryption scheme is secure, the proposed protocol is secure. During evaluation, A learns the predicted result for each sample from B, which is a scalar product, from which A cannot infer B's private information. B learns only the label, from which B cannot infer A's private information.
\end{proof}

At the end of the training process, each party (A or B) only obtains the model parameters associated with its own features, and remains oblivious to the data structure of the other party. At inference time, the two parties need to collaboratively compute the prediction results. Note that the protocol does not deal with the situation in which one (or both) of the two parties is (are) malicious. If A fakes its inputs and submits only one non-zero input, it might be able to tell the value of $u_i^B$ at the position of that input. It still cannot tell $x_i^B$ or $\Theta_B$, and neither party can obtain the correct prediction results.

\section{FTL using secret sharing}\label{BT}
Throughout this section, assume that any private value $v$ is shared between the two parties where $A$ keeps $\langle v\rangle_A$ and $B$ keeps $\langle v\rangle_B$ such that $v=\langle v\rangle_A+\langle v\rangle_B.$ To make it possible for the performance to be comparable with the previous construction, assume that $\ell_1(y,\varphi)$ and $\frac{\partial \ell_1}{\partial \varphi}$ can be approximated by the second order Taylor expansion following Equations \eqref{taylorloss} and \eqref{taylorgrad}. So $\mathcal{L},\frac{\partial\mathcal{L}}{\partial \theta_\ell^A}$ and $\frac{\partial\mathcal{L}}{\partial \theta_\ell^B}$ can be expressed as the following

\begin{equation}\label{polyloss}
\begin{split}
\mathcal{L}&=\sum_i^{N_C}\ell_1(y_i^A,0)+\frac{1}{2}C(y_i^A)\Phi^A\mathcal{G}(u_i^B)\\
&+\left(\frac{1}{8} D(y_i^A)\Phi^A\mathcal{G}(u_i^B)\right)\left(\Phi^A\mathcal{G}(u_i^B)\right)\\
&+\gamma\sum_i^{N_{AB}}\left(\ell_2^B(u_i^B) + \ell_2^A(u_i^A) + \kappa u_i^A(u_i^B)^\prime\right)\\
&+\frac{\lambda}{2} \left(\mathcal{L}_3^A +\mathcal{L}_3^B\right)
\end{split}
\end{equation}

\begin{equation}\label{polygradb}
\begin{split}
\frac{\partial\mathcal{L}}{\partial\theta_\ell^B} &=\sum_i^{N_C} \frac{1}{2} C(y_i^A)\Phi^A \frac{\partial \mathcal{G}(u_i^B)}{\partial \theta_\ell^B}\\
&+2\left[\left(\frac{1}{8} D(y_1^A) \Phi^A \mathcal{G}(u_i^B)\right) \left(\Phi^A\frac{\partial\left(\mathcal{G}(u_i^B)\right)}{\partial \theta_\ell^B}\right) \right]\\
&+\sum_i^{N_{AB}}\left(\gamma \kappa u_i^A\frac{\partial u_i^B}{\partial \theta_\ell^B}+\gamma\frac{\partial\ell_2^B(u_i^B)}{\partial \theta_\ell^B}\right)+\lambda\theta_\ell^B
\end{split}
\end{equation}

\begin{equation}\label{polygrada}
\begin{split}
\frac{\partial\mathcal{L}}{\partial\theta_l^A} &=\sum_i^{N_{C}}\frac{1}{2} C(y_i^A) \frac{\partial\Phi^A}{\partial \theta_\ell^A}\mathcal{G}(u_i^B)\\
&+2\left[\left(\frac{1}{8}D(y_i^A)\Phi^A\mathcal{G}(u_i^B)\right)\left(\frac{\partial \Phi^A}{\partial\theta_\ell^A}\mathcal{G}(u_i^B)\right)\right]\\
&+\gamma\sum_i^{N_{AB}}\left(\kappa u_i^B\frac{\partial u_i^A}{\partial \theta_\ell^A}+\frac{\partial \ell_2^A(u_i^A)}{\partial \theta_\ell^A}\right) +\lambda\theta_\ell^A.
\end{split}
\end{equation}

In this case, the whole process can be performed securely if secure matrix addition and multiplication can be constructed. Since operations with public matrices or adding two private matrices can simply be done using the shares without any communication, the remaining operation that requires discussion is secure matrix multiplication. Beaver's triples are used to help in the matrix multiplication.

\subsection{Secure Matrix Multiplication using Beaver Triples}
 First, we briefly recall how to perform the matrix multiplication given that the two parties have already shared a Beaver's triple. Suppose that the computation required is to obtain $P=MN$ where the dimensions of $M,N$ and $P$ are $m\times n, n\times k$ and $m\times k$ respectively. As assumed, the matrices $M$ and $N$ have been secretly shared by the two parties where $A$ keeps $\langle M\rangle_A$ and $\langle N\rangle_A$ and $B$ keeps $\langle M\rangle_B$ and $\langle N\rangle_B.$ To assist with the calculation, in the preprocessing phase, $A$ and $B$ have generated three matrices $D,E,F$ of dimension $m\times n,n\times k$ and $m\times k$ respectively where $A$ keeps $\langle D\rangle_A, \langle E\rangle_A$ and $\langle F\rangle_A$ while $B$ keeps $\langle D\rangle_B,\langle E\rangle_B$ and $\langle F\rangle_B$ such that $DE=F.$

\begin{algorithm} [ht]
\caption{Secure Matrix Multiplication}
\begin{algorithmic}[1]\label{algo_BTMM}
\REQUIRE $M,N$ two matrices to be multiplied with dimensions $m\times n$ and $n\times k$ respectively and secretly shared between $A$ and $B.$ Triple $(D,E,F=DE)$ of matrices with dimension $m\times n, n\times k$ and $m\times k$ respectively secretly shared between $A$ and $B.$

\STATE $\boldsymbol A$ do:
\STATE ${\langle\delta\rangle_A} \xleftarrow{} \langle M\rangle_A - \langle D\rangle_A,{\langle\gamma\rangle_A} \xleftarrow{} \langle N\rangle_A - \langle E\rangle_A$ and send them to $B$;
\STATE $\boldsymbol B$ do:
\STATE ${\langle\delta\rangle_B} \xleftarrow{} \langle M\rangle_B - \langle D\rangle_B,{\langle\gamma\rangle_B} \xleftarrow{} \langle N\rangle_B - \langle E\rangle_B$ and send them to $A$;
\STATE $\boldsymbol A$ and $\boldsymbol B$ recovers $\delta = \langle\delta\rangle_A+\langle\delta\rangle_B$ and $\gamma = \langle\gamma\rangle_A+\langle\gamma\rangle_B.$
\STATE $\boldsymbol A$ do:
\STATE $\langle P\rangle_A\xleftarrow{} \langle M\rangle_A \gamma + \delta \langle N\rangle_A+ \langle F\rangle_A$;
\STATE $\boldsymbol B$ do:
\STATE $\langle P\rangle_B\xleftarrow{} \langle M\rangle_B \gamma + \delta \langle N\rangle_B+ \langle F\rangle_B-\delta\gamma$;
\end{algorithmic}
\end{algorithm}

It is easy to see that $\langle P\rangle_A + \langle P\rangle_B = MN,$ which is what is required for this protocol. As can be seen, this method guarantees efficient online computation in the cost of having offline phase where players generated the Beavers triples. So next we discuss the scheme to generate the triples.

\subsection{Beaver Triples Generation}

In the preprocessing phase, the Beaver's triples generation protocol uses a sub-protocol to perform the secure matrix multiplication with the help of a third party, which we will call Carlos. Recall that having two matrices $U$ and $V$ owned respectively by $A$ and $B,$ they want to calculate $UV$ securely with the help of Carlos. In order to do this, Alice and Bob individually generate shares for $U$ and $V$ respectively. That is, we have $U=\langle U\rangle_0 + \langle U\rangle_1$ and $V=\langle V\rangle_0 + \langle V\rangle_1.$ Then we have $UV= (\langle U\rangle_0 + \langle U\rangle_1)\cdot(\langle V\rangle_0 + \langle V\rangle_1)= \langle U\rangle_0\langle V\rangle_0 + \langle U\rangle_0\langle V\rangle_1 + \langle U\rangle_1\langle V \rangle_0 +\langle U\rangle_1\langle V\rangle_1.$ So if Alice sends $\langle U\rangle_1$ to Bob and Bob sends $\langle V\rangle_0$ to Alice:
\begin{enumerate}
\item $\langle U\rangle_0\langle V\rangle_0$ can be privately calculated by Alice
\item $\langle U\rangle_1\langle V\rangle_1$ can be privately calculated by Bob
\item $\langle U\rangle_1\langle V\rangle_0$ can be privately calculated by both Alice and Bob
\item However, no one can calculate $\langle U\rangle_0\langle V\rangle_1$ yet. This is what Carlos will calculate.
\end{enumerate}

\begin{algorithm} [ht]
\caption{Offline Secure Matrix Multiplication}
\begin{algorithmic}[1]\label{algo_OffMM}
\REQUIRE $U$ and $V,$ two matrices to be multiplied with dimensions $m\times n$ and $n\times k$ respectively; $U$ is owned by $A$ and $V$ is owned by $B.$
\STATE Invite a third party $C;$
\STATE $\boldsymbol A$ do:
\STATE Randomly choose $\langle U\rangle_0$ and set $\langle U\rangle_1=U-\langle U\rangle_0$;
\STATE Send $\langle U\rangle_1$ to $B$ and $\langle U\rangle_0$ to $C$;
\STATE $\boldsymbol B$ do:
\STATE Randomly choose $\langle V\rangle_0$ and set $\langle V\rangle_1=V-\langle V\rangle_0$
\STATE Send $\langle V\rangle_0$ to $A$ and $\langle V\rangle_1$ to $C$;
\STATE $\boldsymbol C$ do:
\STATE Compute $\tilde{W}=\langle U\rangle_0\langle V\rangle_1$;
\STATE Randomly choose $\langle \tilde{W}\rangle_A$ and set $\langle \tilde{W}\rangle_B=\tilde{W}-\langle \tilde{W}\rangle_A$;
\STATE Send $\langle \tilde{W}\rangle_A$ to $A$ and $\langle\tilde{W}\rangle_B$ to $B$;
\STATE $\boldsymbol A$ do:
\STATE Set $\langle W\rangle_A=\langle U\rangle_0\langle V\rangle_0 + \langle U\rangle_1\langle V\rangle_0 + \langle \tilde{W}\rangle_A$;
\STATE $\boldsymbol B$ do:
\STATE Set $\langle W\rangle_B = \langle U\rangle_1\langle V\rangle_1 + \langle \tilde{W}\rangle_B$;
\end{algorithmic}
\end{algorithm}

By the use of Algorithm \ref{algo_OffMM}, Algorithm \ref{algo_BTG} generates triple $(D,E,F)$ such that:
\begin{enumerate}
\item $DE=F$
\item Alice holds $\langle D\rangle_A,\langle E\rangle_A$ and $\langle F\rangle_A$ without learning anything about $(D,E,F), \langle D\rangle_B, \langle E\rangle_B$ and $\langle F\rangle_B.$
\item Bob holds $\langle D\rangle_B,\langle E\rangle_B$ and $\langle F\rangle_B$ without learning anything about $(D,E,F), \langle D\rangle_A, \langle E\rangle_A$ and $\langle F\rangle_A.$
\end{enumerate}

\begin{algorithm} [ht]
\caption{Beaver Triples Generation}
\begin{algorithmic}[1]\label{algo_BTG}
\REQUIRE The dimensions of the required matrices, $m\times n, n\times k$ and $m\times k$;
\STATE $\boldsymbol A$ do:
\STATE Randomly choose $\langle D\rangle_A$ and $\langle E\rangle_A$ ;
\STATE $\boldsymbol B$ do:
\STATE Randomly choose $\langle D\rangle_B$ and $\langle E\rangle_B$ ;
\STATE $\boldsymbol A$ and $\boldsymbol B$ do:
\STATE Perform Algorithm \ref{algo_OffMM} with $U=\langle D\rangle_A$ and $V=\langle E\rangle_B$ to get $W=\langle D\rangle_A\langle E\rangle_B$ such that $A$ holds $\langle W\rangle_A$ and $B$ holds $\langle W\rangle_B$;
\STATE Perform Algorithm \ref{algo_OffMM} with the role of $A$ and $B$ reversed, $U=\langle D\rangle_B$ and $V=\langle E\rangle_A$ to get $Z=\langle D\rangle_1\langle E\rangle_0$ such that $A$ holds $\langle Z\rangle_A$ and $B$ holds $\langle Z\rangle_B$;
\STATE $\boldsymbol A$ do:
\STATE Set $\langle F\rangle_A = \langle D\rangle_A\langle E\rangle_A + \langle W\rangle_A + \langle Z\rangle_A$;
\STATE $\boldsymbol B$ do:
\STATE Set $\langle F\rangle_B = \langle D\rangle_B\langle E\rangle_B + \langle W\rangle_B + \langle Z\rangle_B$;
\end{algorithmic}
\end{algorithm}
Lastly, during offline phase, Alice and Bob also requested Carlos to generate sufficient number of shares for zero matrices with various dimensions.

\subsection{FTL Algorithm - Secret Sharing based}

Before discussing our FTL protocol that is constructed based on Beaver triples, we first give some notation to simplify Equations \eqref{polyloss}, \eqref{polygradb} and \eqref{polygrada} based on the parties needed to complete the calculation.
\begin{enumerate}
\item Let $\mathcal{L}_A=\sum_i^{N_C} \ell_1(y_i^A,0)+\gamma\sum_{i}^{N_{AB}} \ell_2^A(u_i^A) + \frac{\lambda}{2}\mathcal{L}_3^A,$
\item Let $\mathcal{L}_B=\gamma \sum_i^{N_{AB}}\ell_2^B(u_i^B)+\frac{\lambda}{2}(\mathcal{L}_3^B),$
\item Let
\begin{equation*}
\begin{split}
\mathcal{L}_{AB} &= \frac{1}{2}\sum_i^{N_C} C(y_i^A)\Phi^A\mathcal{G}(u_i^B)\\
&+\frac{1}{8}\left(D(y_i^A)\Phi^A\mathcal{G}(u_i^B)\right)\left(\Phi^A\mathcal{G}(u_i^B)\right)\\
&+\gamma\kappa\sum_{i}^{N_{AB}} u_i^A(u_i^B)^\prime
\end{split}
\end{equation*}
with
\begin{itemize}
\item For $A:$
\begin{itemize}
\item $\mathcal{L}_{AB}^{(A,1)}=\left(\frac{1}{2}C(y_i^A)\Phi^A\right)_{i=1,\cdots, N_C},$
\item $\mathcal{L}_{AB}^{(A,2)}=\left(\frac{1}{8}D(y_i^A)\Phi^A\right)_{i=1,\cdots, N_C}$
\item $\mathcal{L}_{AB}^{(A,3)}=\left(\Phi^A\right)_{i=1,\cdots, N_C}$ and
\item $\mathcal{L}_{AB}^{(A,4)}=(\gamma \kappa u_i^A)_{i=1,\cdots, N_{AB}}.$
\end{itemize}
\item For $B:$
\begin{itemize}
\item $\mathcal{L}_{AB}^{(B,1)}=\left(\mathcal{G}(u_i^B)\right)_{i=1,\cdots, N_C},$
and
\item $\mathcal{L}_{AB}^{(B,2)}=(u_i^B)_{i=1,\cdots,N_C}.$
\end{itemize}
\end{itemize}
Then

\begin{equation*}
\label{splitpolyloss}
\begin{split}
\mathcal{L}&=\mathcal{L}_A + \mathcal{L}_B + \sum_i^{N_C} \mathcal{L}_{AB}^{(A,1)}(i)\mathcal{L}_{AB}^{(B,1)}(i)\\
&+ \left(\mathcal{L}_{AB}^{(A,2)}(i)\mathcal{L}_{AB}^{(B,1)}(i)\right)\left(\mathcal{L}_{AB}^{(A,3)}(i)\mathcal{L}_{AB}^{(B,1)}(i)\right)\\
&+ \sum_i^{N_{AB}} \mathcal{L}_{AB}^{(A,4)}(i)\left(\mathcal{L}_{AB}^{(B,2)}(i)\right)^\prime.
\end{split}
\end{equation*}


\item Let $D^{(B,\ell)}_B=\sum_i^{N_{AB}}\gamma\frac{\partial \ell_2^B(u_i^B)}{\partial \theta_\ell^B}+\lambda \theta_\ell^B,$
\item Let

\begin{equation*}
\begin{split}
D^{(B,\ell)}_{AB} &= \sum_i^{N_C} \frac{1}{2} C(y_i^A)\Phi^A \frac{\partial \mathcal{G}(u_i^B)}{\partial \theta_\ell^B}\\
&+2{(\frac{1}{8}D(y_i^A)\Phi^A)\mathcal{G}(u_i^B)(\Phi^A)(\frac{\partial \mathcal{G}(u_i^B)}{\partial \theta_\ell^B})}\\
&+\sum_i^{N_{AB}} \gamma\kappa u_i^A\frac{\partial u_i^B}{\partial \theta_\ell^B}
\end{split}
\end{equation*}


with
\begin{itemize}
\item For $B$:
\begin{itemize}
\item $D_{AB,B,1}^{(B,\ell)}=\left(\frac{\partial\mathcal{G}(u_i^B)}{\partial \theta_\ell^B}\right)_{i=1,\cdots,N_C},$
\item $D_{AB,B,2}^{(B,\ell)}=\left(\frac{\partial u_i^B}{\partial\theta_\ell^B}\right)_{i=1,\cdots, N_{AB}}.$
\end{itemize}
\end{itemize}
Then

\begin{equation*}\label{splitpolygradb}
\begin{split}
\frac{\partial \mathcal{L}}{\partial\theta_\ell^B}&=D_B^{B,\ell}+\sum_i^{N_C}\mathcal{L}_{AB}^{(A,1)}(i) D_{AB,B,1}^{(B,\ell)}(i)\\
&+2{\left(\mathcal{L}_{AB}^{(A,2)}(i)\mathcal{L}_{AB}^{(B,1)}(i)\right)\left(\mathcal{L}_{AB}^{(A,3)}(i)D_{AB,B,1}^{(B,\ell)}(i)\right)}\\
&+\sum_i^{N_{AB}} \mathcal{L}_{AB}^{(A,4)}(i)D_{AB,B,2}^{(B,\ell)}(i)
\end{split}
\end{equation*}

\item Let $D^{(A,\ell)}_A=\sum_i^{N_{AB}}\gamma\frac{\partial \ell_2^A(u_i^A)}{\partial \theta_\ell^A}+\lambda \theta_\ell^A,$
\item Let

\begin{equation*}
\begin{split}
D^{(A,\ell)}_{AB}&=\sum_i^{N_{C}}\frac{1}{2} C(y_i^A) \frac{\partial\Phi^A}{\partial \theta_\ell^A}\mathcal{G}(u_i^B)\\
&+2\left(\frac{1}{8}D(y_i^A)\Phi^A\right)\mathcal{G}(u_i^B)\left(\frac{\partial \Phi^A}{\partial\theta_\ell^A}\right)\mathcal{G}(u_i^B)\\
&+\gamma\sum_i^{N_{AB}}\left(\kappa u_i^B\frac{\partial u_i^A}{\partial \theta_\ell^A}\right)
\end{split}
\end{equation*}

with
\begin{itemize}
\item For $A:$
\begin{itemize}
\item $D_{AB,A,1}^{(A,\ell)}=\left(\frac{1}{2}C(y_i^A)\frac{\partial\Phi^A}{\partial\theta_\ell^A}\right)_{i=1,\cdots, N_C},$
\item $D_{AB,A,2}^{(A,\ell)}=\left(\frac{\partial\Phi^A}{\partial\theta_\ell^A}\right)_{i=1,\cdots, N_C}$ and
\item $D_{AB,A,3}^{(A,\ell)}=\left(\gamma\kappa\frac{\partial u_i^A}{\partial \theta_\ell^A}\right)_{i=1,\cdots, N_{AB}}.$
\end{itemize}
\end{itemize}
Then


\begin{equation*}\label{splitpolygrada}
\begin{split}
\frac{\partial \mathcal{L}}{\partial\theta_\ell^A}&= D_A^{(A,\ell)}+\sum_i^{N_{C}}D_{AB,A,1}^{(A,\ell)}(i)\mathcal{L}_{AB}^{(B,1)}(i)\\
&+2\left(\mathcal{L}_{A,B}^{(A,2)}\mathcal{L}_{AB}^{(B,1)}(i)\right)\left(D_{AB,A,2}^{(A,\ell)}(i)\mathcal{L}_{AB}^{(B,1)}(i)\right)\\
&+\gamma\sum_i^{N_{AB}}\left(D_{AB,A,3}^{(A,\ell}(i)(u_i^B)'\right)
\end{split}
\end{equation*}

\end{enumerate}

To perform the training scheme, both Alice and Bob first initialize and execute their respective neural networks $Net^A$ and $Net^B$ locally to obtain $u_i^A$ and $u_i^B.$ Alice computes $\{h_k^A(u_i^A,y_i^A)\}_{k=1,\cdots, K_A}.$ Then for each $k,$ she randomly chooses a mask $\langle h_k^A(u_i^A,y_i^A)\rangle_A$ and sets $\langle h_k^A(u_i^A,y_i^A)\rangle_B =h_k^A(u_i^A,y_i^A) - \langle h_k^A(u_i^A,y_i^A)\rangle_A.$ Then Alice sends $\langle h_k^A(u_i^A,y_i^A)\rangle_B$ to Bob for $k=1,\cdots, K_A.$ Similarly, Bob computes $\{h_k^B(u_i^B)\}_{k=1,\cdots, K_B}$ and for each $k,$ he randomly chooses $\langle h_k^B(u_i^B)\rangle_B$ and sets $\langle h_k^B(u_i^B)\rangle_A = h_k^B(u_i^B)-\langle h_k^B(u_i^B)\rangle_B$, which is then sent to Alice.

In our scenario, $K_A=7,$ with:
\begin{enumerate}
\item $h_1^A(u_i^A,y_i^A) = \mathcal{L}_{AB}^{(A,1)},$
\item $h_2^A(u_i^A,y_i^A) = \mathcal{L}_{AB}^{(A,2)},$
\item $h_3^A(u_i^A,y_i^A) = \mathcal{L}_{AB}^{(A,3)},$
\item $h_4^A(u_i^A,y_i^A) = \mathcal{L}_{AB}^{(A,4)},$
\item $h_5^A(u_i^A,y_i^A) = \mathcal{D}_{AB,A,1}^{(A,\ell)},$
\item $h_6^A(u_i^A,y_i^A) = \mathcal{D}_{AB,A,2}^{(A,\ell)},$
\item $h_7^A(u_i^A,y_i^A) = \mathcal{D}_{AB,A,3}^{(A,\ell)}$
\end{enumerate}
and $K_B=4$ with
\begin{enumerate}
\item $h_1^B(u_i^B) = \mathcal{L}_{AB}^{(B,1)},$
\item $h_2^B(u_i^B) = \mathcal{L}_{AB}^{(B,2)},$
\item $h_3^B(u_i^B) = \mathcal{D}_{AB,B,1}^{(B,\ell)}$
\item $h_4^B(u_i^B) = \mathcal{D}_{AB,B,2}^{(B,\ell)}.$
\end{enumerate}

In addition, Alice privately computes $\mathcal{L}_A$ and $D_A^{(A,\ell)}$ while Bob privately computes $\mathcal{L}_B$ and $D_B^{(B,\ell)}.$ Algorithm \ref{algo_BTPred} provides the training protocol for one iteration based on Beaver triples generated by Algorithm \ref{algo_BTG}.

\begin{algorithm} [ht]
\caption{FTL Training: Beaver triples based}
\begin{algorithmic}[1]\label{algo_BTT}
\REQUIRE Alice holds $h_k^A, \mathcal{L}_A$ and $D_A^{(A,\ell)}$ while Bob holds $h_k^B, \mathcal{L}_B$ and $D_B^{(B,\ell)}.$ In the offline phase, they have also generated sufficient triples with the appropriate dimensions. We also require a threshold $\epsilon$ for termination condition;
\STATE Calculate $\mathcal{L}_{AB}$ with $3N_C+N_{AB}$ inner products of length $d$ and two real number multiplications. Alice receives $\langle \mathcal{L}_{AB}\rangle_A$ and Bob receives $\langle \mathcal{L}_{AB}\rangle_B.$ Alice sets $\langle \mathcal{L}\rangle_A = \mathcal{L}_A + \langle \mathcal{L}_{AB}\rangle_A$ and Bob sets $\langle \mathcal{L}\rangle_B = \mathcal{L}_B+\langle \mathcal{L}_{AB}\rangle_B$;
\STATE Both Alice and Bob publish their shares so they can individually recover $\mathcal{L}$;
\STATE For each $\theta_\ell^B\in \Theta^B,$ calculate $D_{AB}^{B,\ell}$ with $3N_C+N_{AB}$ inner product of vectors of length $d$ and two real number multiplications. Alice receives $\langle D_{AB}^{(B,\ell)}\rangle_A$ and sets $\left\langle\frac{\partial\mathcal{L}}{\partial \theta_\ell^B}\right\rangle_A = \langle D_{AB}^{(B,\ell)}\rangle_A.$ In the same time, Bob receives $\langle D_{AB}^{(B,\ell)}\rangle_B$ and sets $\left\langle\frac{\partial\mathcal{L}}{\partial \theta_\ell^B}\right\rangle_B = D_B^{(B,\ell)}+\langle D_{AB}^{(B,\ell)}\rangle_B$;
\STATE Alice sends $\left\langle\frac{\partial\mathcal{L}}{\partial \theta_\ell^B}\right\rangle_A$ to Bob;
\STATE Bob recovers $\frac{\partial\mathcal{L}}{\partial \theta_\ell^B}=\left\langle\frac{\partial\mathcal{L}}{\partial \theta_\ell^B}\right\rangle_A+\left\langle \frac{\partial\mathcal{L}}{\partial \theta_\ell^B}\right\rangle_B$;
\STATE Bob updates $\theta_\ell^B$;
\STATE For each $\theta_\ell^A\in \Theta^A,$ calculate $D_{AB}^{A,\ell}$ with $3N_C+N_{AB}$ inner product of vectors of length $d$ and two real number multiplications. Alice receives $\langle D_{AB}^{(A,\ell)}\rangle_A$ and sets $\left\langle\frac{\partial\mathcal{L}}{\partial \theta_\ell^A}\right\rangle_A = D_{A}^{(A,\ell)}+\langle D_{AB}^{(A,\ell)}\rangle_A.$ In the same time, Bob receives $\langle D_{AB}^{(A,\ell)}\rangle_B$ and sets $\left\langle\frac{\partial\mathcal{L}}{\partial \theta_\ell^A}\right\rangle_B = \langle D_{AB}^{(A,\ell)}\rangle_B$;
\STATE Bob sends $\left\langle\frac{\partial\mathcal{L}}{\partial \theta_\ell^A}\right\rangle_B$ to Alice;
\STATE Alice recovers $\frac{\partial\mathcal{L}}{\partial \theta_\ell^A}=\left\langle\frac{\partial\mathcal{L}}{\partial \theta_\ell^A}\right\rangle_A+\left\langle \frac{\partial\mathcal{L}}{\partial \theta_\ell^A}\right\rangle_B$;
\STATE Alice updates $\theta_\ell^A$;
\STATE Bob updates $\theta_\ell^B$;
\STATE Repeat as long as $\mathcal{L}_{prev}-\mathcal{L}\geq\epsilon;$
\end{algorithmic}
\end{algorithm}

After the training is completed, we proceed to the prediction phase. Recall that after the training phase, Alice has the optimal value for $\Theta^A$ while Bob has the optimal value for $\Theta^B.$ Suppose that now $B$ wants to learn the label for $\{x_j^B\}_{j\in N_B}.$ The protocol can be found in Algorithm \ref{algo_BTPred}.

\begin{algorithm} [ht]
\caption{FTL Prediction: Beaver triples based}
\begin{algorithmic}[1]\label{algo_BTPred}
\REQUIRE Alice holds the optimal parameter $\Theta^A$ and Bob holds the optimal parameter $\Theta^B$ and unlabeled data $\{x_j^B\}_{j\in N_B}$;
\STATE $\boldsymbol B$ do:
\STATE Calculate $u_j^B = Net^B(\Theta^B,x_j^B)$;
\STATE Calculate $\mathcal{G}(u_j^B)$;
\STATE Randomly choose $\langle \mathcal{G}(u_j^B)\rangle_B$;
\STATE Set $\langle\mathcal{G}(u_j^B)\rangle_A = \mathcal{G}(u_j^B)-\langle\mathcal{G}(u_j^B)\rangle_B$ and send it to $A$;
\STATE $\boldsymbol A$ do:
\STATE Calculate $\Phi^A$
\STATE Randomly choose $\langle\Phi^A\rangle_A$;
\STATE Set $\langle\Phi^A\rangle_B = \Phi^A - \langle\Phi^A\rangle_A$ and send it to $B$;
\STATE Perform secure matrix multiplication from Algorithm \ref{algo_BTMM} so $A$ receives $\langle \Phi^A\mathcal{G}(u_j^B)\rangle_A$ and $B$ receives $\langle \Phi^A\mathcal{G}(u_j^B)\rangle_B.$
\STATE $\boldsymbol B$ sends $\langle\Phi^A\mathcal{G}(u_j^B)\rangle_B$ to $A$;
\STATE $\boldsymbol A$ recovers $\varphi(u_j^B)=\Phi^A\mathcal{G}(u_j^B),$ calculates $y_j^B$ and sends it to $B$;
\end{algorithmic}
\end{algorithm}

\begin{theorem}\label{BTsecproofthm}
The protocol in Algorithms \ref{algo_OffMM},\ref{algo_BTG},\ref{algo_BTT} and \ref{algo_BTPred} are information theoretically secure against at most one passive adversary.
\end{theorem}
\begin{proof}
Note that in all of these algorithms, the only information that any party receives regarding any private values is only the share for an $n$-out-of-$n$ secret sharing scheme. So by the property of $n$-out-of-$n$ secret sharing scheme, no one can learn any information about the private values they are not supposed to learn. After the calculation, the same thing can be said since each party only learns about a share of a secret sharing scheme and they cannot learn any information regarding values they are not supposed to learn from there.
\end{proof}

\begin{remark}
Using the argument in \cite{DBLP:journals/iacr/MohasselZ17} we can improve the efficiency in the following manner; for each matrix $A,$ it is always masked by the same random matrix. This optimization does not affect the security of the protocol while significantly improves the efficiency.
\end{remark}

\section{Experimental Evaluation}

In this section, we report experiments conducted on public datasets including: 1) NUS-WIDE dataset \cite{Chua09nus-wide:a} 2) Kaggle's \textit{Default-of-Credit-Card-Clients} \cite{kaggle} (``Default-Credit'') to validate our proposed approach. We study the effectiveness and scalability of the approach with respect to various key factors, including the number of overlapping samples, the dimension of hidden common representations, and the number of features.

The NUS-WIDE dataset consists of 634 low-level features from Flickr images as well as their associated tags and ground truth labels. There are in total 81 ground truth labels. We use the top 1,000 tags as text features and combine all the low-level features including color histograms and color correlograms as image features. We consider solving a one-vs-all classification problem with a data federation formed between party A and party B, where A has 1000 text tag features and labels, while party B has 634 low-level image features. 

The ``Default-Credit'' dataset consists of credit card records including user demographics, history of payments, and bill statements, etc., with users' default payments as labels. After applying one-hot encoding to categorical features, we obtain a dataset with 33 features and 30,000 samples. We then split the dataset both in the feature space and the sample space to simulate a two-party federation problem. Specifically, we assign each sample to party A, party B or both so that there exists a small number of samples overlapping between A and B. All labels are on the side of party A. We will examine the scalability (in section \ref{ftl_scalability}) of the FTL algorithm by dynamically splitting the feature space.

\begin{figure*}[ht]
\centering
\subfigure[Learning loss (1-layer)]{\includegraphics[width=0.24\linewidth]{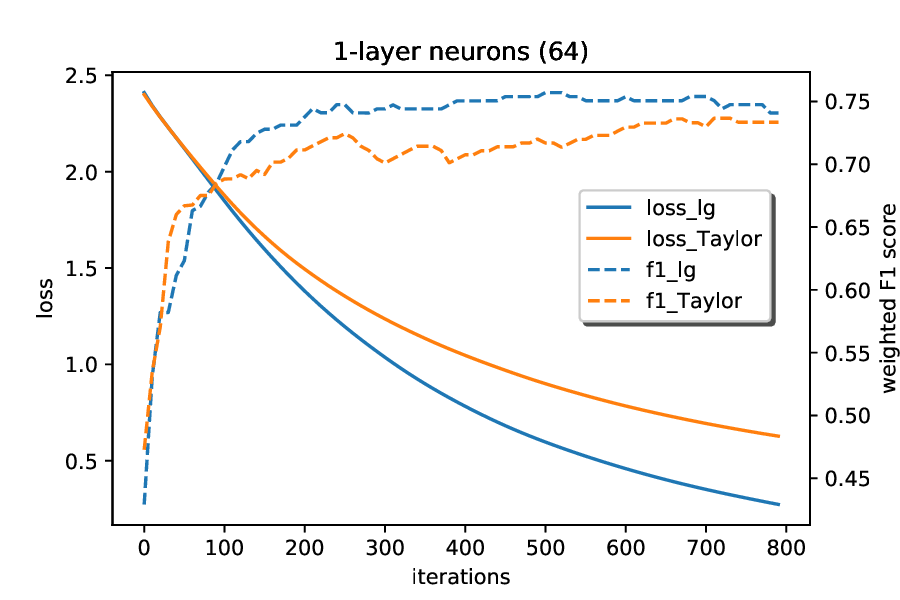}\label{taylor_1}}
\space\space\space\space\space\space\space
\subfigure[Learning loss (2-layer)]{\includegraphics[width=0.24\linewidth]{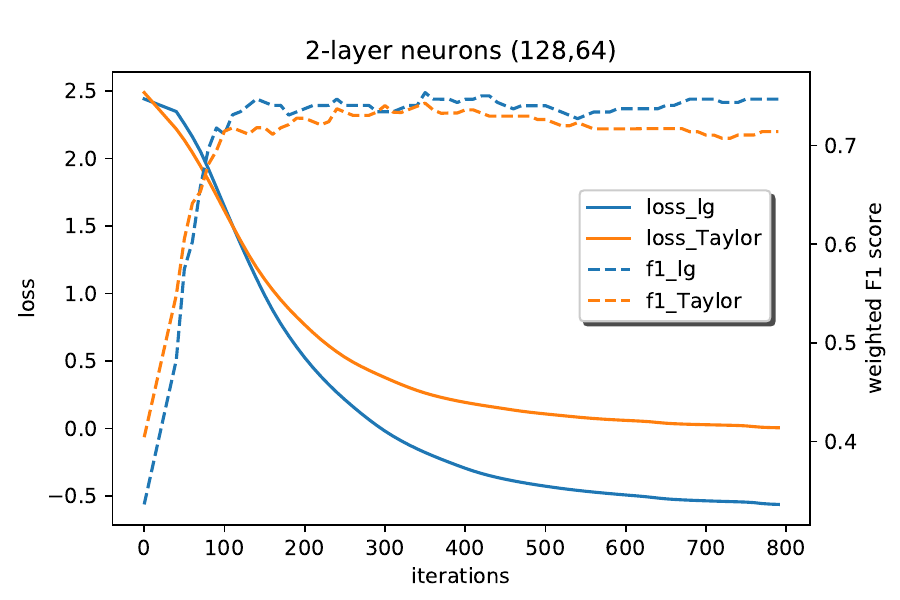}\label{taylor_2}}
\space\space\space\space\space\space\space
\subfigure[F1 vs. \# overlapping pairs]
{\includegraphics[width=0.24\linewidth]{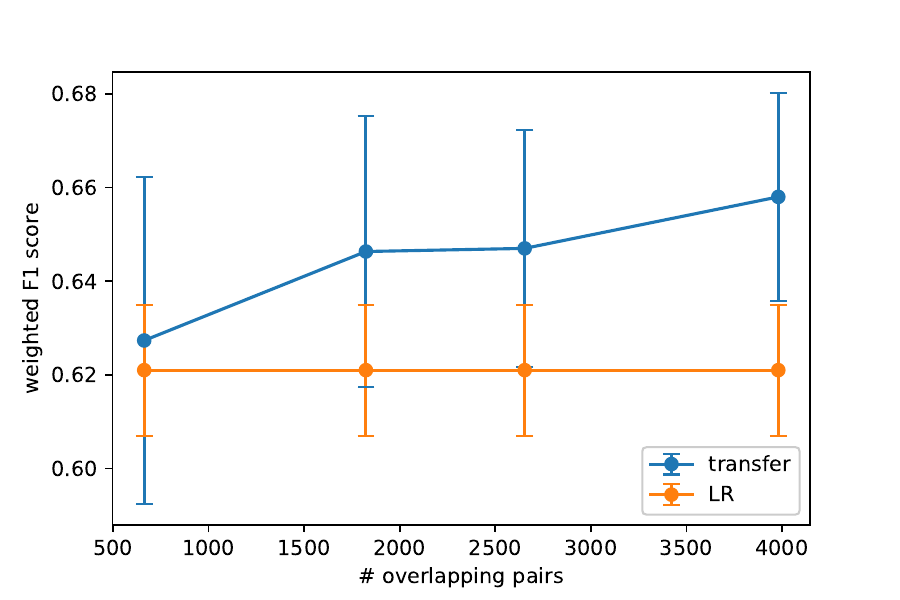}\label{overlap}}
\subfigure[time vs. $d$ (HE)]
{\includegraphics[width=0.24\linewidth]{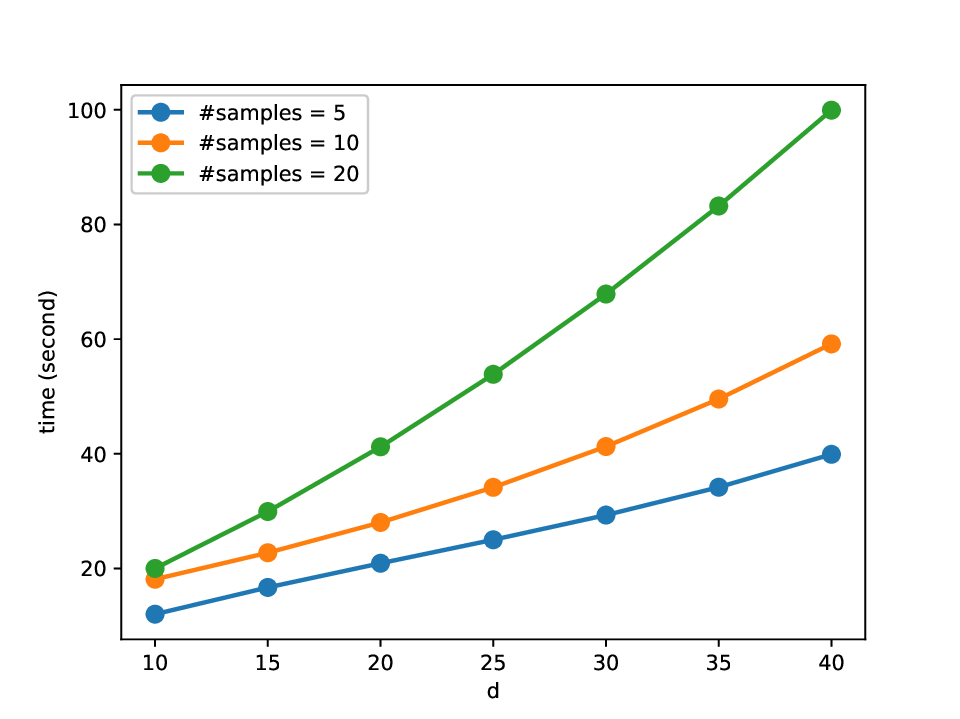}\label{scale_d}}
\space\space\space\space\space\space\space
\subfigure[time vs. \# features (HE)]
{\includegraphics[width=0.24\linewidth]{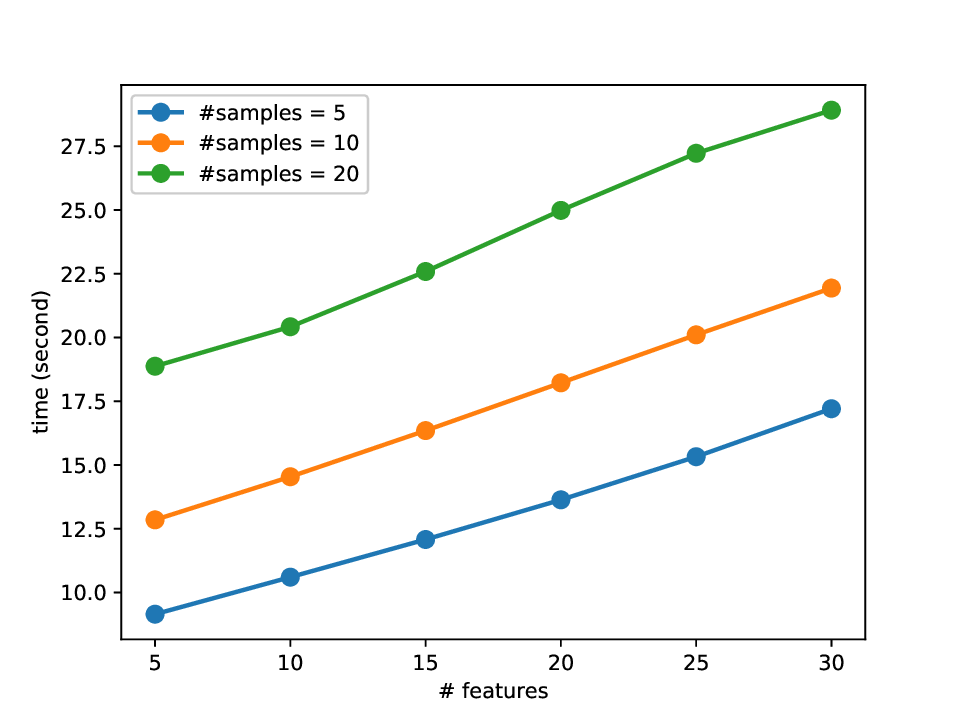}\label{scale_f}}
\space\space\space\space\space\space\space
\subfigure[time vs. \# samples (HE)]
{\includegraphics[width=0.24\linewidth]{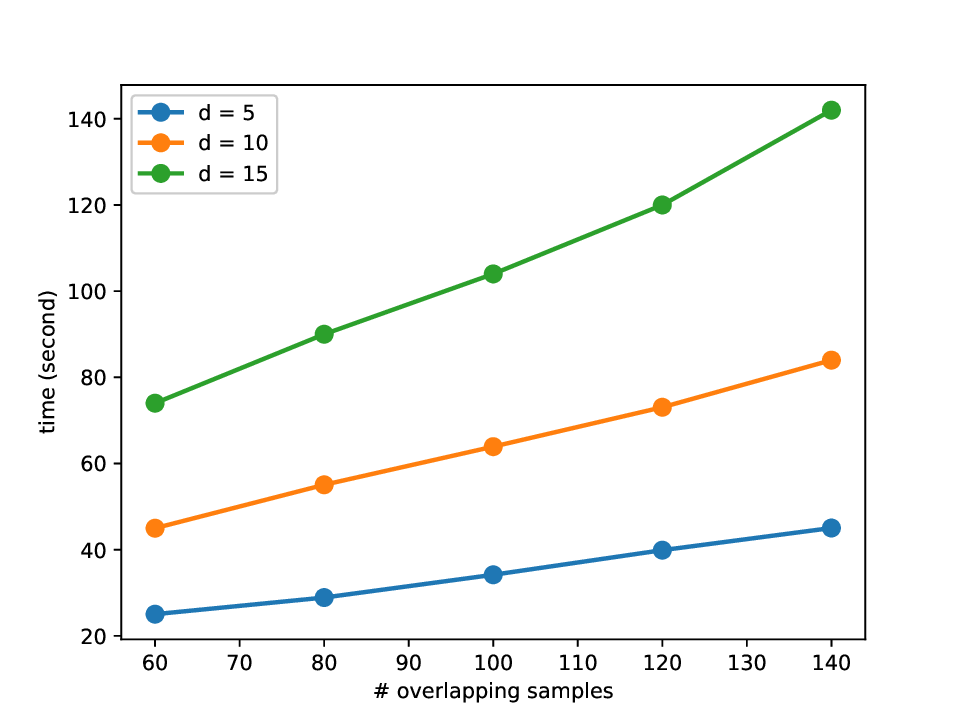}\label{scale_s}}
\subfigure[time vs. $d$ (SS)]
{\includegraphics[width=0.24\linewidth]{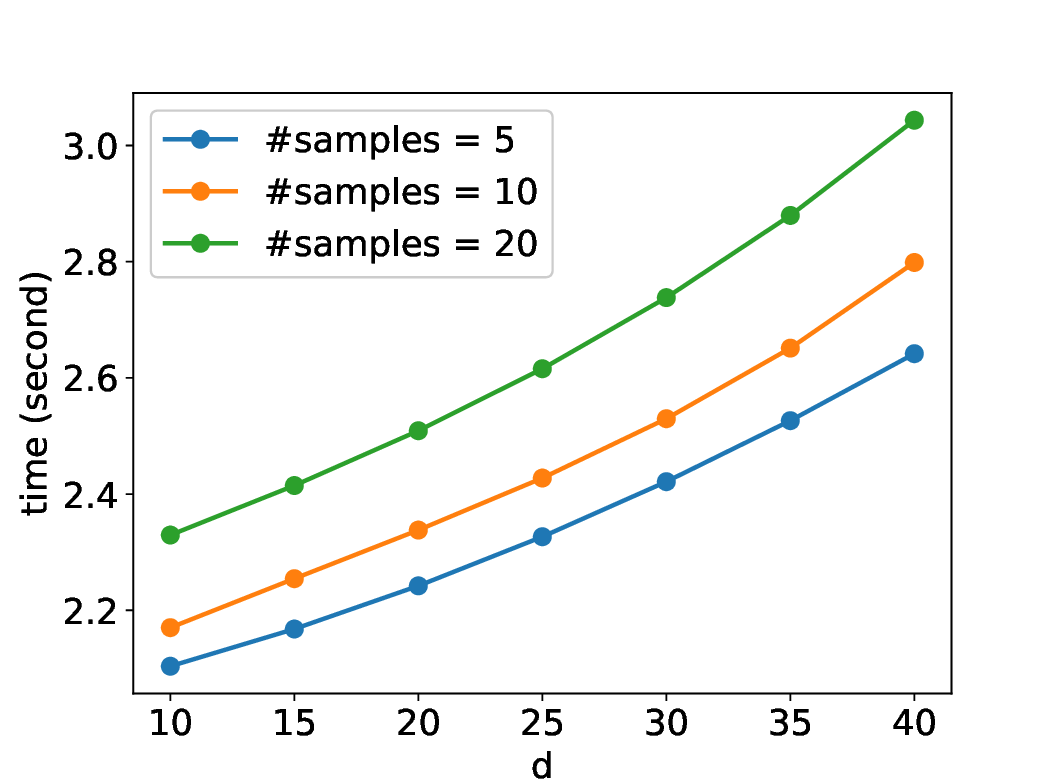}\label{scale_ss_d}}
\space\space\space\space\space\space\space
\subfigure[time vs. \# features (SS)]
{\includegraphics[width=0.24\linewidth]{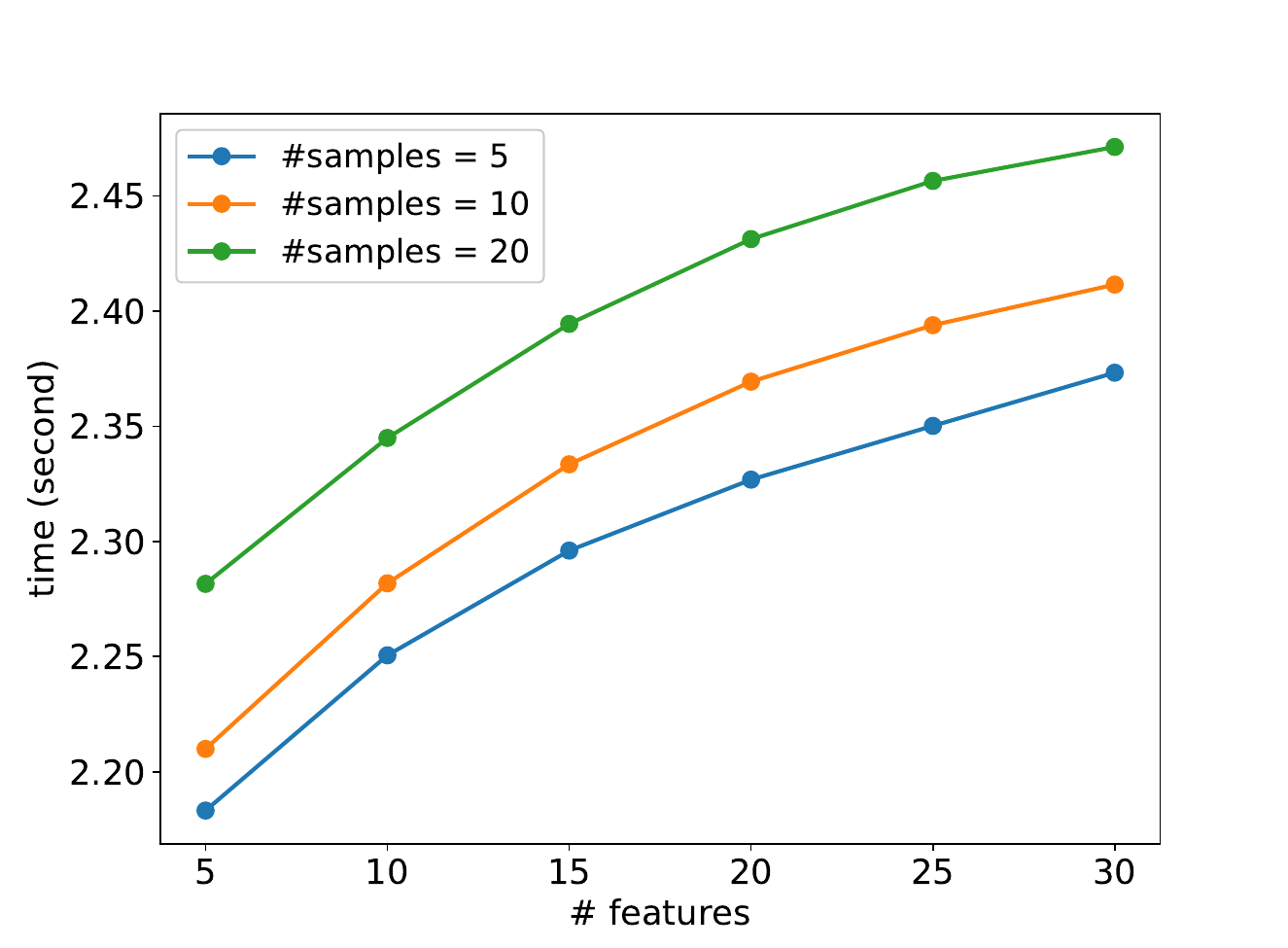}\label{scale_ss_f}}
\space\space\space\space\space\space\space
\subfigure[time vs. \# samples (SS)]
{\includegraphics[width=0.24\linewidth]{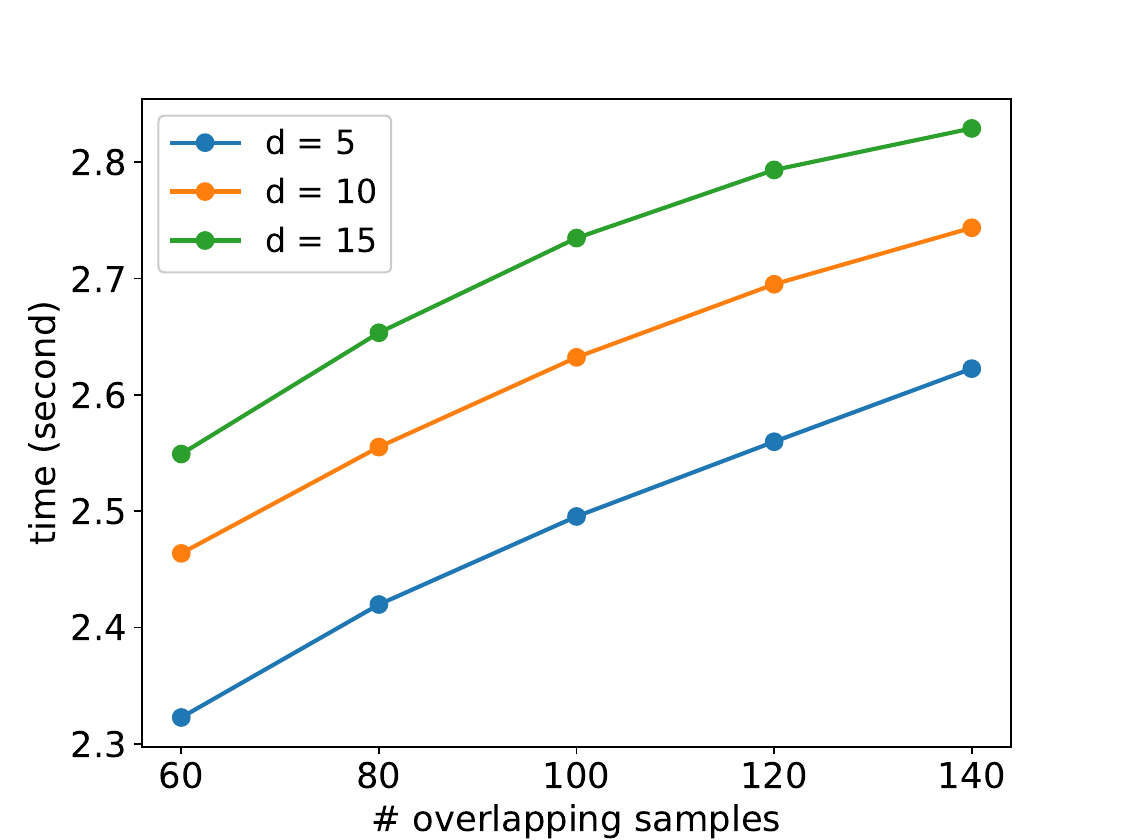}\label{scale_ss_s}}
\caption{Experiment results.}
\label{Taylor}
\end{figure*}

\begin{table*}[ht]\Huge
\centering
\caption{Comparison of weighted F1 scores.}
\label{TLvsLG}
\resizebox{0.80\textwidth}{!}{
    \begin{tabular}{||c|c|c c c c c c||}\hline
    Tasks & $N_c$ & SST & TLT & TLL & LR & SVMs & SAEs\\\hline
    water vs. others &  $100$ & $\boldsymbol{0.698}\pm0.011$ & $0.692\pm0.062$ & $0.691\pm0.060$ & $0.685\pm0.020$ & $0.640\pm0.016$ & $0.677\pm 0.048$ \\\hline
    water vs. others &  $200$ & $\boldsymbol{0.707}\pm0.013$ & $0.702\pm0.010$ & $0.701\pm 0.007$ & $0.672\pm 0.023$ & $0.643 \pm 0.038$ & $0.662 \pm 0.010$ \\\hline
    person vs. others &  $100$ & $\boldsymbol{0.703}\pm0.015$ & $0.697\pm 0.010$ & $0.697 \pm 0.020$ & $0.694 \pm 0.026$ & $0.619 \pm 0.050$ & $0.657 \pm 0.030$\\\hline
    person vs. others & $200$ & $\boldsymbol{0.735}\pm0.004$ & $0.733 \pm 0.009$ & $\boldsymbol{0.735} \pm 0.010$ & $0.720 \pm 0.004$ & $0.706 \pm 0.023$ & $0.707 \pm 0.008$ \\\hline
    sky vs. others & $100$ & $0.708\pm0.015$ & $0.700\pm0.022$ & $\boldsymbol{0.713}\pm 0.006$ & $0.694 \pm 0.016$ & $0.679 \pm 0.018$ & $0.667 \pm 0.009$ \\\hline
    sky vs. others & $200$ & $\boldsymbol{0.724}\pm0.014$ & $0.718 \pm 0.033$ & $0.718 \pm 0.024$ & $0.696 \pm 0.026$ & $0.680 \pm 0.042$ & $0.684 \pm 0.056$ \\\hline
\end{tabular}
}
\end{table*}

\subsection{Impact of Taylor Approximation}

We studied the effect of Taylor approximation by monitoring and comparing the training loss decay and the performance of prediction. Here, we test the convergence and precision of the algorithm using the NUS-WIDE data and neural networks with different levels of depth. In the first case, $Net^A$ and $Net^B$ both have one auto-encoder layer with 64 neurons, respectively. In the second case, $Net^A$ and $Net^B$ both have two auto-encoder layers with 128 and 64 neurons, respectively. In both cases, we used 500 training samples, 1,396 overlapping pairs, and set $\gamma=0.05,\lambda=0.005$. We summarize the results in Figures \ref{taylor_1} and \ref{taylor_2}. We found that the loss decays at a similar rate when using Taylor approximation as compared to using the full logistic loss, and the weighted F1 score of the Taylor approximation approach is also comparable to the full logistic approach. The loss converges to a different minima in each of these cases. As we increased the depth of the neural networks, the convergence and the performance of the model did not decay.


Most existing secure deep learning frameworks suffer from accuracy loss when adopting privacy-preserving techniques \cite{DBLP:journals/iacr/MohasselZ17}. 
Using only low-degree Taylor approximation, the drop in accuracy in our approach is much less than the state-of-art secure neural networks with similarly approximations.

\subsection{Performance}

We tested SS-based FTL (SST), HE-based FTL with Taylor loss (TLT) and FTL with logistic loss (TLL). For the self-learning approach, we picked three machine learning models: 1) logistic regression (LR), 2) SVM, and 3) stacked auto-encoders (SAEs). The SAEs are of the same structure as the ones we used for transfer learning, and are connected to a logistic layer for classification. We picked three of the most frequently occurring labels in the NUS-WIDE dataset, i.e., water, person and sky, to conduct one vs. others binary classification tasks. For each experiment, the number of overlapping samples we used is half of the total number of samples in that category. We varied the size of the training sample set and conducted three tests for each experiment with different random partitions of the samples. The parameters $\lambda$ and $\gamma$ are optimized via cross-validation.

Figure \ref{overlap} shows the effect of varying the number of overlapping samples on the performance of transfer learning. The overlapping sample pairs are used to bridge the hidden representations between the two parties. The performance of FTL improves as the overlap between datasets increases.

The comparison of F-score (mean $\pm$ std) among SST, TLT, TLL and the several other machine learning models is shown in Table \ref{TLvsLG}. We observe that SST, TLT and TLL yield comparable performance across all tests. This demonstrates that SST can achieve plain-text level accuracy while TLT can achieve almost lossless accuracy although Taylor approximation is applied. The three FTL models outperform baseline self-learning models significantly using only a small set of training samples under all experimental conditions. In addition, performance improves as we increased the number of training samples. The results demonstrated the robustness of FTL.

\subsection{Scalability}\label{ftl_scalability}


We study the scalability using Default-Credit dataset because it allows us to conveniently choose features when we do experiments. Specifically, we study how the training time scales with the number of overlapping samples, the number of target-domain features, and the dimension of hidden representations, denote as $d$. Based on the algorithmic detail of proposed transfer learning approach, the communication cost for B sending a message to A can be calculated by formula $Cost_{B\xrightarrow{}A} = n*(d^2+d)*ct$, where $ct$ is the size of the message and $n$ is the number of samples sent. The same cost applies when sending message from A to B. 

To speed up the secure FTL algorithm, we preform compute-intensive operations in parallel. The logic flow of parallel secure FTL algorithm includes three stages: parallel encryption, parallel gradient calculation, and parallel decryption. Detailed logic flow is shown in Figure \ref{parallel_flow}.

\begin{figure}[!ht]
\centering
\label{SeqVsParallel}
\includegraphics[width=0.90\linewidth, height=0.25\textheight]{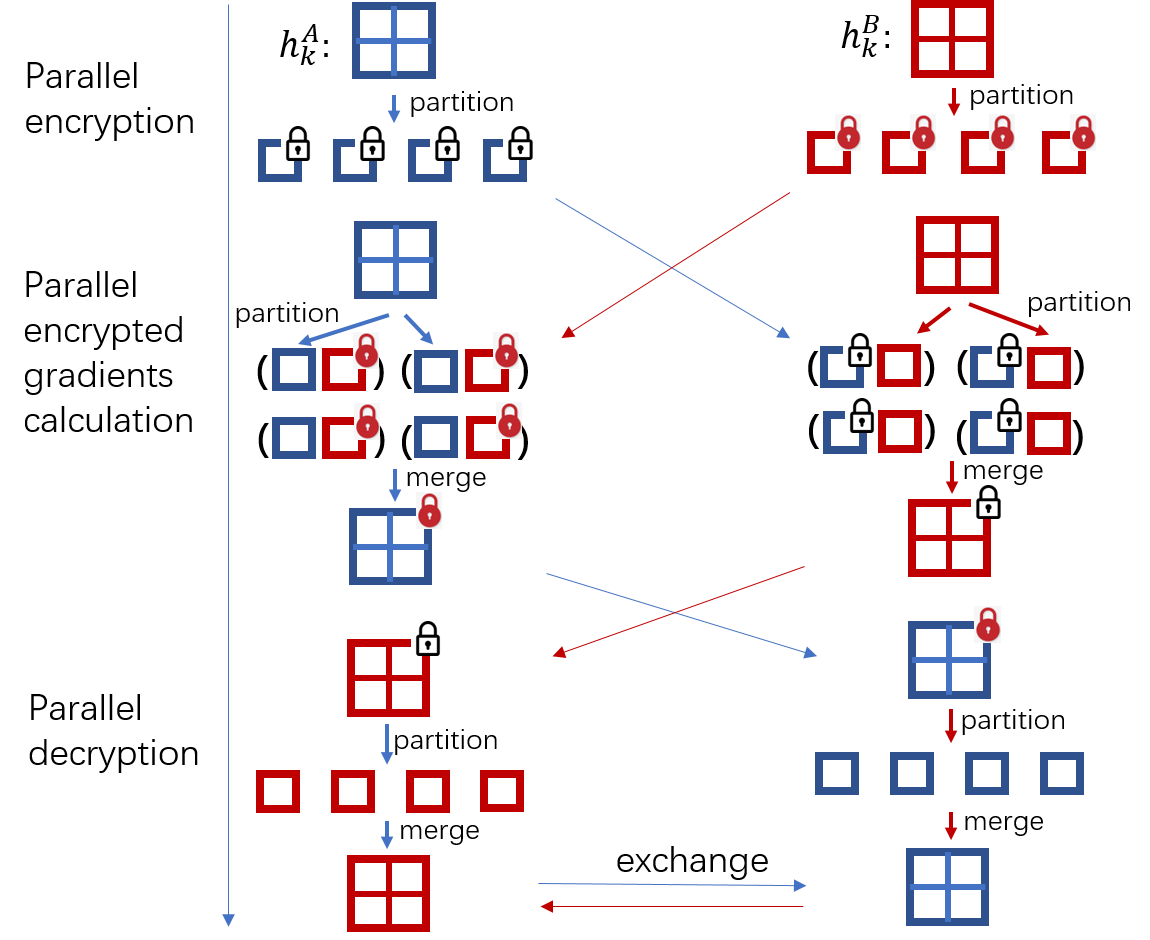}
\caption{Logic flow of parallel secure FTL.} \label{parallel_flow}
\end{figure}

On parallel encryption stage, we parallelly encrypt components that will be sent to the other party. On parallel gradient calculation stage, we parallelly perform operations, including matrix multiplication and addition, on encrypted components to calculate encrypted gradients. On parallel decryption stage, we parallelly decrypt masked loss and gradients. Finally, the two parties exchange decrypted masked gradients that will be used to update neural networks. With 20 partitions, the parallel scheme can boost the secure FTL 100x than sequential scheme.


Figures \ref{scale_d}, \ref{scale_f} and \ref{scale_s} illustrate that with parallelism applied, the running time of HE-based FTL grows approximately linearly with respect to the size of the hidden representation dimension, the number of target-domain features, as well as the number of overlapping samples respectively. 



Figures  \ref{scale_ss_d}, \ref{scale_ss_f} and \ref{scale_ss_s} illustrate how the training time varies with the three key factors in the SS setting. The communication cost can be simplified as $O(d^2)$ if keeping other factors constant. As illustrated in Figures \ref{scale_ss_d}, however, the increasing rate of the training time is approaching linear rather than $O(d^2)$. We conjecture that this is due to the computational efficiency of SS-based FTL. Besides, as illustrated in Figure \ref{scale_ss_f} and \ref{scale_ss_s}, respectively, as the feature sizes or overlapping samples increase, the increasing rate of training time drops.



\begin{table}[ht]
\centering
\caption{Comparison of training time between SS and HE with the increasing dimension of hidden representation denoted by $d$, the increasing number of target-domain features, and the increasing number of overlapping samples, respectively}\label{table_f}
\resizebox{0.48\textwidth}{!}{
\begin{tabular}{|l||*{6}{c|}}\hline
\backslashbox{protocol}{$d$}
&\makebox[2em]{15}&\makebox[2em]{20}&\makebox[2em]{25}&\makebox[2em]{30}&\makebox[2em]{35}&\makebox[2em]{40}\\\hline\hline
HE training time (sec) & 29.12 & 41.03 & 53.88 & 66.74 & 81.91 & 101.02 \\\hline
SS training time (sec) &  2.41 &  2.52 &  2.61 &  2.73 &  2.89 &   3.04 \\\hline
\hline
\hline
\backslashbox{protocol}{\# features}
&\makebox[2em]{5}&\makebox[2em]{10}&\makebox[2em]{15}&\makebox[2em]{20}&\makebox[2em]{25}&\makebox[2em]{30}\\\hline\hline
HE training time (sec) & 17.82 & 20.45 & 21.67 & 24.03 & 27.12 & 28.58 \\\hline
SS training time (sec) &  2.28 &  2.35 &  2.39 &  2.43 &  2.46 & 2.48 \\\hline
\hline
\hline
\backslashbox{protocol}{\# samples}&\makebox[2em]{60}&\makebox[2em]{80}&\makebox[2em]{100}&\makebox[2em]{120}&\makebox[2em]{140} & \\\hline\hline
HE training time (sec) & 74.21 & 89.91 & 111.12 & 123.79 & 146.48 &\\\hline
SS training time (sec) & 2.55 & 2.65 & 2.74 &  2.79 & 2.83 &\\\hline
\end{tabular}
}
\end{table}

Further, we compare the scalability of SS-based with that of HE-based FTL along the axis of the hidden representation dimension, the number of features, and the number of overlapping samples, respectively. The results are presented in Table \ref{table_f}. We notice that SS-based FTL is running much faster than HE-based FTL. Overall, SS-based FTL speeds up by 1-2 orders of magnitude compared with HE-based FTL. In addition, as shown in the three tables, the increasing rate of the training time of SS-based FTL is much slower than that of HE-based FTL.


\section{Conclusions and Future Work}

In this paper we proposed a secure Federated Transfer Learning (FTL) framework to expand the scope of existing secure federated learning to broader real-world applications. Two secure approaches, namely, homomorphic encryption (HE) and secret sharing are proposed in this paper for preserving privacy. The HE approach is simple, but computationally expensive. The biggest advantages of the secret sharing approach include (i) there is no accuracy loss, (ii) computation is much faster than HE approach. The major drawback of the secret sharing approach is that one has to offline generate and store many triplets before online computation.

We demonstrated that, in contrast to existing secure deep learning approaches which suffer from accuracy loss, FTL is as accurate as non-privacy-preserving approaches, and is superior to non-federated self-learning approaches. The proposed framework is a general privacy-preserving federated transfer learning solution that is not restricted to specific models.

In future research, we will continue improving the efficiency of the FTL framework by using distributed computing techniques with less expensive computation and communication schemes.

\bibliographystyle{IEEEtran}
\bibliography{ref}
\end{document}